%% file: main.tex
\documentclass[letterpaper, 10 pt, conference]{ieeeconf}
\IEEEoverridecommandlockouts
\usepackage[utf8]{inputenc}

\usepackage{amsmath}
\usepackage{amsfonts}
\usepackage{amssymb}
\usepackage{gensymb}
\usepackage{xcolor}
\usepackage{graphicx}
\usepackage{caption}
\usepackage{subcaption}
\usepackage[ruled]{algorithm2e}
\newtheorem{theorem}{Theorem}
\newtheorem{assumption}{Assumption}
\newtheorem{remark}{Remark}

\DeclareMathOperator*{\argmin}{arg\,min}
\usepackage{mathtools}
\usepackage{tikz}
\usepackage{environ}
\usetikzlibrary{positioning}
\usepackage{bm}

\usepackage{cite}

\usepackage[font={footnotesize,it}]{caption}
\newcommand{\Obs}[1]{\mathcal{O}^{#1}}

\title{\LARGE \bf Obstacle Avoidance in Dynamic Environments via Tunnel-following MPC with Adaptive Guiding Vector Fields}

\author{Albin Dahlin, Yiannis Karayiannidis% <-this % stops a space
\thanks{This work has been supported by Chalmers AI Research Centre (CHAIR) and AB Volvo through the project AiMCoR.}%
\thanks{ A. Dahlin is with the Department of Electrical Engineering, Chalmers University of Technology, SE-412 96 Gothenburg, Sweden 
        {\tt\small albin.dahlin@chalmers.se}}%
 \thanks{ Y. Karayiannidis is with the Department of Automatic Control, LTH, Lund University,  SE-221 00 Lund, Sweden. Y. K. is a member of the ELLIIT Strategic Research Area at Lund University.
        {\tt\small yiannis@control.lth.se}}%
}

\begin{document}

\maketitle
\thispagestyle{empty}
\pagestyle{empty}

\begin{abstract}
    This paper proposes a motion control scheme for robots operating in a dynamic environment with concave obstacles. 
    A Model Predictive Controller (MPC) is constructed to drive the robot towards a goal position while ensuring collision avoidance without direct use of obstacle information in the optimization problem. This is achieved by guaranteeing tracking performance of an appropriately designed receding horizon path. The path is computed using a guiding vector field defined in a subspace of the free workspace where each point in the subspace satisfies a criteria for minimum distance to all obstacles.
    The effectiveness of the control scheme is illustrated by means of simulation. 
\end{abstract}

\section{Introduction}
Navigating autonomous agents to a goal position in a dynamic environment with both moving obstacles, such as humans and other autonomous systems, and static obstacles is a common problem in robotics. The resulting trajectory must be collision-free and obey possible robot constraints. Traditionally, motion planning problems have been solved considering static maps, but dynamic environments with moving obstacles require online adjustments of the planned robot path to avoid crashes. A common method to tackle such problems is to construct closed form control laws that result in closed loop reactive dynamical systems (DS) that possess desirable stability and convergence properties. Specifically, artificial potential fields \cite{khatib_85}, repelling the robot from the obstacles, have become popular \cite{ginesi_etal_19,stavridis_etal_17}. 
However, a drawback of the additive potential field methods is that they may yield local minimum other than the goal point, i.e. the robot could get stuck at a position away from the goal. To address this issue, navigation functions \cite{rimon_koditschek_92,loizou_11_2,paternain_etal_18} and harmonic potential fields have emerged \cite{connolly_etal_90, feder_slotine_97, daily_bevly_08, huber_etal_19, huber_etal_22}. A repeated assumption in the aforementioned methods enabling the proof of (almost) global convergence is the premise of disjoint obstacles. 
However, intersecting obstacles are frequently occurring, e.g. when modelling complex obstacles as a combination of several simpler shapes, or when the obstacle regions are padded to take robot radius or safety margins into account.
In \cite{dahlin_karayiannidis_22} a workspace modification algorithm was presented to obtain a workspace of disjoint obstacles such that the convergence properties of the aforementioned DS methods are preserved. 

With the increase of computational power and development of robust numerical solvers for optimization problems, optimization-based techniques, such as Model Predictive Control (MPC), have become popular. Compared to the closed form control laws, MPC allows for an easy encoding of the system constraints. MPC is typically used as a local planner given a global reference path or waypoints which are computed based on the static environment. All dynamic obstacles are accommodated in the MPC formulation. Commonly, the obstacle regions (or approximation of the regions) are explicitly expressed in the optimization problem, either as hard constraints \cite{schulman_etal_14,zhang_etal_21,brito_etal_19} or soft constraints by including a penalizing term in the cost function \cite{sanchez_etal_21,ji_etal_16}. Due to the receding horizon nature of the MPC, these works do not provide convergence guarantees. Specifically, in environments with large obstacles, or where intersecting obstacles creates concave regions, the MPC solution may lead to local attractors at obstacle boundaries.
While many MPC formulations focus on trajectory tracking, given a reference trajectory, path-following MPC \cite{faulwasser_findeisen_16} gives highest priority to the minimization of the robot deviation from some geometric reference path, with less focus on the velocity profile. Tunnel-following MPC \cite{vanduijkeren_19} extends the path-following MPC scheme by imposing a constraint on tracking error, restricting the robot to stay within some specified distance to the reference path.

In this work, we present a motion control scheme which combines a DS method for receding horizon path generation with a tunnel-following MPC. In this way, an attracting behavior towards the goal with ensured collision avoidance is obtained although no explicit obstacle constraints are used in the optimization problem formulation. In particular, the number of constraints for the inner control loop is independent of the number and shape of obstacles.
In contrast to a pure closed form approach, embedding the closed form DS scheme in an MPC scheme allows for simple adaptation of the robot constraints to find an admissible and smooth control input. Compared to other MPC approaches, collision avoidance is here achieved by relying on a reference path generator, simplifying the formulation of the optimal control problem to be independent of workspace complexity.
Overall, the contribution of the proposed approach is summarized below:
\begin{itemize}
    \item An MPC framework that allows realization of DS-generated trajectories when they are not obeying system constraints. 
    \item The MPC solver is guaranteed to provide existence of collision-free solutions at all times. 
    \item The formulation of the optimal control problem in the proposed MPC is independent of workspace complexity.
\end{itemize}

\section{Preliminaries}

\subsection{Starshaped sets and star worlds}
A set $A\subset \mathbb{R}^n$ is \textit{starshaped with respect to} (w.r.t.) $x$ if for every point $y\in A$ the line segment $l(x,y)$ is contained by $A$. The set $A$ is said to be \textit{starshaped} if it is starshaped w.r.t.  some point, i.e. $\exists x$ s.t. $l(x,y) \subset A, \forall y \in A$. 
The set $A$ is \textit{strictly starshaped w.r.t. $x$} if it is starshaped w.r.t. $x$ and any ray emanating from $x$ crosses the boundary only once. We say that $A$ is strictly starshaped if it is strictly starshaped w.r.t. some point.

Given a collection of obstacles $\mathcal{O} = \{\Obs{1}, \Obs{2}, ...\}$ in $\mathbb{R}^n$, the free space $\mathcal{F} = \mathbb{R}^n \setminus \bigcup_{\Obs{j}\in\mathcal{O}}\Obs{j}$ is said to be a \textit{star world} if all obstacles are strictly starshaped. A star world where all obstacles are mutually disjoint is defined as a \textit{disjoint star world}. For more information on starshaped sets and star worlds, see \cite{hansen_etal_20} and \cite{dahlin_karayiannidis_22}.

\subsection{Obstacle avoidance for dynamical systems in star worlds}
\label{sec:soads}
Given a star world, obstacle avoidance can be achieved using a DS approach\cite{huber_etal_19} with dynamics:% of the following form:
\begin{equation}
    \label{eq:ds_obs_avoidance}
    \dot{r} = \eta(r,r^g,\mathcal{O}) = M(r,\mathcal{O})(r^g-r),
\end{equation}
where $\mathcal{O}$ is the collection of strictly starshaped obstacles forming the star world, $\mathcal{F}^{\star}$, $r$ is the current robot position and $r^g\in \mathcal{F}^{\star}$ is the goal position. $M(\cdot,\cdot)$ is a modulation matrix used to adjust the attracting dynamics to $r^g$ based on the obstacles tangent spaces.
Convergence to $r^g$ is guaranteed for a trajectory following \eqref{eq:ds_obs_avoidance} from any initial position $r^0\in\mathcal{F}^{\star}$ if $\mathcal{F}^{\star}$ is a disjoint star world and no obstacle center point is contained by the line segment $l(r^0,r^g)$. For more information, see \cite{huber_etal_19,huber_etal_22}.

In \cite{dahlin_karayiannidis_22} the authors presented a method to establish a disjoint star world $\mathcal{F}^{\star} \subset \mathcal{F}$ from a free space, $\mathcal{F}$, formed by possibly intersecting convex and/or polygon obstacles. In essence, the algorithm combines all clusters of intersecting obstacles and extend the resulting obstacle regions such that they are strictly starshaped and mutually disjoint. The center points of the reshaped obstacles are placed outside the line segment $l(r^0,r^g)$ to satisfy the condition for convergence of the dynamics \eqref{eq:ds_obs_avoidance}.
The algorithm is not complete in the sense that there may be cases where it does not find a solution when such in fact exists. In case no disjoint star world is found, the algorithm returns a (intersecting) star world satisfying $\mathcal{F}^{\star} = \mathcal{F}$, and obstacle avoidance guarantees for \eqref{eq:ds_obs_avoidance} are remained while the convergence property is not obtained.

\section{Problem formulation}
Consider a robot of radius $a$ operating in the Cartesian plane with discrete-time dynamics for a sampling interval, $\Delta t$, given as
\begin{equation}
\label{eq:robot_model}
\begin{split}
    x_{k+1} &= f(x_k,u_k)\\
    p_k &= h(x_k).
\end{split}
     %= \begin{bmatrix}v(t)\cos(\theta(t))\\v(t)\sin(\theta(t))\\\omega(t) \end{bmatrix}
\end{equation}
Here, $x_k \in\mathcal{X}\subset \mathbb{R}^n$ is the robot state, $p_k\in \mathbb{R}^2$ is the robot position and $u_k\in \mathcal{U}\subset \mathbb{R}^m$ is the control signal at time instance $k$. It is assumed that there exists a control input such that the robot does not move, i.e. $\forall x\in\mathcal{X},\ \exists u'\in \mathcal{U} \textnormal{ s.t. } f(x,u') = x$. The robot is operating in a world containing a collection of dynamic, possibly intersecting, obstacles, $\tilde{\mathcal{O}}_k = \{\tilde{\mathcal{O}}^1_k, \tilde{\mathcal{O}}^2_k, ...\}$, which are either convex shapes or polygons. 
\begin{remark}
Although $\tilde{\mathcal{O}}_k$ formally contains only convex shapes and polygons, the formulation allows for more general complex obstacles as intersections are allowed. In particular, any shape can be described as a combination of several convex and polygon regions.
\end{remark}
No future information of obstacle movement is available.
To take into account the robot radius we define the dilated obstacles as $\mathcal{O}_k = \{\tilde{\mathcal{O}}^j_k \oplus \mathbb{B}[\mathbf{0}, a]\}_{\tilde{\mathcal{O}}^j_k \in \tilde{\mathcal{O}}_k}$, where $\oplus$ is the Minkowski sum and $\mathbb{B}[\mathbf{0}, a]$ is the closed ball of radius $a$ centered at $\mathbf{0}=[0,0]^T$. The free space that includes all collision-free robot positions is then given as $\mathcal{F}_k = \mathbb{R}^2 \setminus \bigcup_{\Obs{j}_k \in \mathcal{O}_k}\Obs{j}_k$.

\begin{assumption}
\label{ass:quasi_static_obstacles}
The obstacle move slow compared to the sampling frequency, i.e. over a control sampling period, $\Delta t$, the obstacle positions are constant.
% The obstacles are moving in a discrete-time manner over a control sampling period, $\Delta t$.%, {\color{red}such that $\mathcal{O}(t) = \mathcal{O}_k,\ t \in [k\Delta t,(k+1)\Delta t), k\in \mathbb{N}$ for some static set $\mathcal{O}_k$}.
\end{assumption}
\begin{assumption}
\label{ass:not_aggressive_obstacles}
The obstacles do not actively move into a region occupied by the robot, such that the implication $p_{k+1}\in\mathcal{F}_k \Rightarrow p_{k+1}\in\mathcal{F}_{k+1}$ holds.
\end{assumption}

The objective is to find a control policy that enforces the robot to stay in the free set at all times, $p_k\in\mathcal{F}_k,\ \forall k$, and drives it to a goal position $p^g\in\mathbb{R}^2$.

In the following sections, we will omit the time notation for convenience unless some ambiguity exists.

\begin{figure}
    \centering
    \resizebox{\linewidth}{!}{
        \input{figures/architecture}
    }
    \caption{Proposed motion control scheme.}
    \label{fig:architecture}
\end{figure}
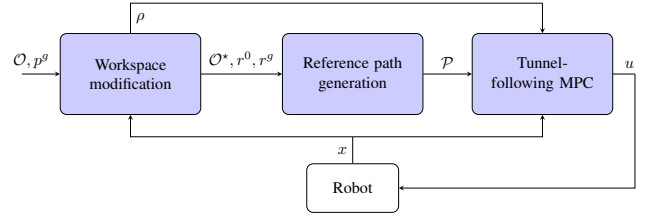

\section{Control design}
\label{sec:method}
We propose a motion control scheme depicted in Fig. \ref{fig:architecture} which consists of three main components. First, the obstacles are modified to form a disjoint star world. The star world is designed as a strict subspace of the free space with any interior point having an appropriately selected minimum clearance to the obstacles. Next, a DS approach which ensures obstacle avoidance and convergence to the goal within disjoint star worlds is utilized to generate a receding horizon reference path. Finally, a tunnel-following MPC is used to compute a control sequence which ensures close path tracking such that collision avoidance guarantees are obtained.
Details are given in the following subsections.

\subsection{Workspace modification}
\label{sec:workspace_modification}
The proposed method relies on generating a reference path with a (time-varying) minimum clearance, $\rho$, to all obstacles using the DS approach \eqref{eq:ds_obs_avoidance}. To this end, the set of inflated obstacles $\mathcal{O}^{\rho} = \{\Obs{j} \oplus \mathbb{B}(\mathbf{0}, \rho)\}_{\Obs{j} \in \mathcal{O}}$ is defined, see Fig. \ref{fig:sim_static_rho}, with corresponding clearance set $\mathcal{F}^{\rho} = \mathbb{R}^2 \setminus \bigcup_{\Obs{j}\in\mathcal{O}^{\rho}}\Obs{j} \subset \mathcal{F}$. Here, $\mathbb{B}(\mathbf{0}, \rho)$ is the open ball of radius $\rho$ centered at $\mathbf{0}$. That is, $\mathcal{F}^{\rho}$ is the set of all robot positions where the closest distance to an obstacle is at least $\rho$. 
As stated in Section \ref{sec:soads}, any star world is positively invariant for the dynamics \eqref{eq:ds_obs_avoidance} and convergence to a goal position is guaranteed for a disjoint star world. Since $\mathcal{O}^{\rho}$ may include both intersecting and non-starshaped obstacles, $\mathcal{F}^{\rho}$ provides none of the aforementioned guarantees. Thus, the objective of the workspace modification is to find a disjoint star world $\mathcal{F}^{\star}\subset\mathcal{F}^{\rho}$ containing an initial position, $r^0$, and a goal position, $r^g$, for the reference path. A procedure to specify $\rho$ and to compute $\mathcal{O}^{\star}$, $r^0$ and $r^g$ is given in Algorithm \ref{alg:obstacle_transformation} and the steps are elaborated below.

\LinesNumbered
\begin{algorithm}
\caption{Workspace modification}
\label{alg:obstacle_transformation}
\SetKwRepeat{Do}{do}{while}
\SetKwInOut{Input}{Input}
\SetKwInOut{Output}{Output}
\SetKwInOut{Parameters}{Parameters}
\SetKwProg{Init}{init}{}{}
\Parameters{$\gamma$, $\bar{\rho}$}
\Input{$\mathcal{O}$, $p^g$, $p$}
\Output{$\mathcal{O}^{\star}, r^0, r^g, \rho$}
$\rho \gets \bar{\rho} / \gamma$ \label{l:rho0}\;
\Do{$\mathcal{P}^0 = \emptyset$ \label{l:Pi_nonempty}}{
$\rho \gets \gamma\rho$\;
$\mathcal{P}^0 \gets \mathbb{B}(p, \rho) \setminus \{\Obs{j} \oplus \mathbb{B}(\mathbf{0}, \rho)\}_{\Obs{j}\in \mathcal{O}}$\;
}
$\mathcal{O}^{\rho} \gets \{\Obs{j} \oplus \mathbb{B}[\mathbf{0}, \rho]\}_{\Obs{j}\in \mathcal{O}}$\;
$r^0 = \argmin_{r^0\in \mathcal{P}^0} \lVert r^0 - p \lVert_2$ \label{l:ri}\;
$r^g = \argmin_{r^g\not\in \mathcal{O}^{\rho}} \lVert r^g - p^g \lVert_2$\label{l:rg}\;
Compute $\mathcal{O}^{\star}$ using Algorithm 2 in \cite{dahlin_karayiannidis_22} with $\mathcal{O}^{\rho}$, $r^0$ and $r^g$\label{l:O_star}\;
\ForEach{$\Obs{j} \in \mathcal{O}^{\star}$ \label{l:ch_init}}{
\If{$CH(\Obs{j}) \cap \left\{r^0 \cup r^g \cup \bigcup_{\Obs{l}\in\mathcal{O}^{\star}\setminus\Obs{j}}\Obs{l} \right\} = \emptyset$}{
$\Obs{j} \gets CH(\Obs{j})$
}
}\label{l:ch_end}
\end{algorithm}

\subsubsection{Clearance selection (line \ref{l:rho0}-\ref{l:Pi_nonempty})}
The initial reference position is chosen within the initial reference set $\mathcal{P}^0 = \mathcal{F}^{\rho} \cap \mathbb{B}(p, \rho)$, depicted as green area in Fig. \ref{fig:sim_static_rho}. To have a valid initial reference position, $\rho$ is set to a strict positive value such that $\mathcal{P}^0$ is nonempty. Such a $\rho$ is guaranteed to exist for any collision-free robot position\footnote{$\mathcal{F}$ is an open set, so $\exists \rho>0$ s.t. $\mathbb{B}(p,\rho)\in\mathcal{F}$. Thus, the closest distance to an obstacle is at least $\rho$, i.e. $p \in \mathcal{F}^{\rho}$, and it follows that $p\in\mathcal{P}^0 \Rightarrow \mathcal{P}^0\neq\emptyset$.}, $p\in \mathcal{F}$. In Algorithm \ref{alg:obstacle_transformation}, $\rho$ is initially set to a base value $\bar{\rho}$ and reduced by a factor $\gamma$ until $\mathcal{P}^0\neq\emptyset$. Here, $\bar{\rho}\in \mathbb{R}^+$ and $\gamma\in(0,1)$ are algorithm parameters.

\subsubsection{Initial and goal reference position selection (line \ref{l:ri}-\ref{l:rg})}
The reference path should ideally be a curve from the current robot position, $p$, to the goal, $p^g$. 
However, since $\mathcal{F}^{\rho}$ is a strict subset of $\mathcal{F}$ it is possible that $p \not\in \mathcal{F}^{\rho}$ or $p^g \not\in \mathcal{F}^{\rho}$. In particular, this occurs when the robot or goal position is located closer than a distance $\rho$ to an obstacle. To account for these situations, we define the initial reference position $r^0 = \argmin_{r^0\in \mathcal{P}^0} \lVert r^0 - p \lVert_2$ and reference goal $r^g = \argmin_{r^g\in \mathcal{F}^{\rho}} \lVert r^g - p^g \lVert_2$.

\subsubsection{Establishment of a disjoint star world (line \ref{l:O_star})}
Using Algorithm 2 from \cite{dahlin_karayiannidis_22}, a disjoint star world $\mathcal{F}^{\star}\subset\mathcal{F}^{\rho}$ is constructed based on the inflated obstacles, $\mathcal{O}^{\rho}$, such that $r^0\in\mathcal{F}^{\star}$ and $r^g\in\mathcal{F}^{\star}$.

\subsubsection{Convexification (line \ref{l:ch_init}-\ref{l:ch_end})}
While convergence to the goal position is guaranteed following the dynamics \eqref{eq:ds_obs_avoidance} for any disjoint star world, the behaviour is not always the most intuitive in concave regions. To obtain a more direct path, any concave obstacle is made convex using the convex hull provided this does not violate the two conditions: 1) $r^0$ and $r^g$ remain exterior points of the obstacle, and 2) the resulting obstacle region does not intersect with any other obstacle. When the obstacles are made convex, unnecessary "detours" in concave regions are avoided, compare Fig. \ref{fig:sim_static_ch} and \ref{fig:sim_static_goal}. 

\subsection{Reference path generation}
The reference path is given as a parameterized regular curve
\begin{equation}
\label{eq:P_star}
    \mathcal{P} = \left\{r\in \mathbb{R}^2 : s \in [0, N] \rightarrow r(s) \right\}
\end{equation}
where $N\in\mathbb{N}^+$ is a parameter determining the path horizon and $r$ is given by the solution to the ODE
\begin{equation}
\label{eq:reference_dynamics}
    \frac{dr(s)}{ds} = \Delta p_{\max}\bar{\eta}(r(s),r^g, \mathcal{O}^{\star}),\quad r(0) = r^0.
\end{equation}
Here, $\bar{\eta}(\cdot,\cdot,\cdot)=\frac{\eta(\cdot,\cdot,\cdot)}{\lVert \eta(\cdot,\cdot,\cdot) \rVert_2}$ are the normalized dynamics in \eqref{eq:ds_obs_avoidance}
and $\Delta p_{\max} = \max_{u\in \mathcal{U},x\in\mathcal{X}}\lVert \frac{\partial h}{\partial x}(x)f(x,u) \rVert_2$ is the maximum linear displacement which can be achieved by the robot in one sampling instance. The use of the normalized dynamics is instrumental for the MPC problem formulation.
Since $\bar{\eta}$ is a vector of unit length, the arc length of the receding horizon path, $\mathcal{P}$, is $\Delta p_{\max}N$ unless the dynamics \eqref{eq:reference_dynamics} converge to $r^g$ before this distance is reached.
As the path is initialized in the star world $\mathcal{F}^{\star}$ and the dynamics are positively invariant in any star world, we have $\mathcal{P}\subset \mathcal{F}^{\star}$. Moreover, since $\mathcal{F}^{\star}\subset \mathcal{F}^{\rho}$, any point in $\mathcal{P}$ is at least at a distance $\rho$ from any obstacle in $\mathcal{O}$.
Assuming $\mathcal{F}^{\star}$ has successfully been constructed as a disjoint star world, $r$ is guaranteed to converge towards $r^g$, i.e. $\lim_{N\rightarrow \infty} r(N) = r^g$.

\subsection{Tunnel-following MPC}
To find a control input which drives a robot with dynamics \eqref{eq:robot_model} along the reference path \eqref{eq:P_star}, a nonlinear MPC is formulated. The objective is to find a solution which result in a fast movement while staying close enough to the reference path such that collision avoidance is obtained. 

To derive a trajectory from the reference path, the state is extended with the path coordinate $s$ and a virtual control signal $\Delta s$ is introduced which determines the path coordinate increment, i.e. $s_{i+1}=s_i+\Delta s_i$. 
The MPC is formulated based on the extended state $z_i=[\bar{x}^T_i\ s_i]^T$ with dynamics
\begin{equation}
\label{eq:z_dyn}
    z_{i+1} = \begin{bmatrix}f(\bar{x}_i,\bar{u}_i)\\ s_i + \Delta s_i \end{bmatrix} = f_z(z_i,\mu_i).
\end{equation}
where $\mu_i = [\bar{u}^T_i\ \Delta s_i]^T$ is the extended control signal. Here, we have used the notation $\bar{x}_i$ and $\bar{u}_i$ to distinguish the internal variables of the controller from the real system variables.
According to \eqref{eq:P_star}, the path coordinate is restricted to $[0,N]$ and we have the state constraint $z_i\in \mathcal{Z}=\mathcal{X}\times [0,N]$. The control input is constrained by $\mu_i \in \mathcal{M} = \mathcal{U} \times [0, 1]$. The lower bound on $\Delta s$ is chosen to ensure a forward motion of the reference trajectory along $\mathcal{P}$ and the upper bound is set to $1$ corresponding to a reference position movement equal to the maximal achievable linear displacement of the robot.
Similar to a tunnel-following MPC scheme \cite{vanduijkeren_19}, we impose a constraint on the tracking error, $\varepsilon_i = r(s_i)-h(\bar{x}_i)$, such that the robot position is in a $\rho$-neighborhood of the reference position\footnote{In contrast to \cite{vanduijkeren_19} we apply strict, and not soft, constraints on the tracking error. This can be done and still ensure existence of solution from the design of the reference path. In particular, since $r(0)\in \mathcal{P}^0\subset \mathbb{B}(p,\rho)$.}. That is, the tracking error constraint is $\varepsilon \in \mathcal{E}=\mathbb{B}(\mathbf{0},\rho)$. 
The optimization problem for the MPC is proposed as follows:
\begin{subequations}
\begin{align}
&\!\min_{\bm{\mu}} & &
-c_s s_N + c_e\varepsilon^T_N\varepsilon_N\\
&\text{s.t.} & & z_0 = [x^T, 0]^T \label{eq:mpc_init} \\
&i\in\mathcal{N}: &\  & z_{i+1} = f_z(z_i,\mu_i) \label{eq:mpc_dyn} \\
&  &      & \mu_i \in \mathcal{M} \label{eq:mpc_mu}\\ 
&  &      & z_{i+1} \in \mathcal{Z} \label{eq:mpc_z}\\
&  &      & \varepsilon_{i+1} \in \mathcal{E} \label{eq:mpc_error}
\end{align}
\label{eq:mpc}
\end{subequations}
\noindent where $\mathcal{N}=\{0,..,N-1\}$ and $\bm{\mu}=\{\mu_i : i\in\mathcal{N}\}$ is used to denote the control sequence over the horizon. The positive scalars $c_s$ and $c_e$ are tuning parameters.
The cost function is designed to motivate a solution where the robot moves in a forward direction of $\mathcal{P}$. This is done by maximizing the final path coordinate, $s_N$, or equivalently since $s_N=\sum_{i\in\mathcal{N}}\Delta s_i$, maximizing path increment at each time instance. That is, the reference position is desired to move along $\mathcal{P}$ at a fast rate. This in turn drives the robot position in the same direction due to the coupling effect of the tracking error constraint \eqref{eq:mpc_error}. At the same time, \eqref{eq:mpc_error} restricts the reference position from diverging from the robot position at robot configurations when the linear velocity in the tangential direction of $\mathcal{P}$ is limited. 
To boost forward motion further, and to provide converging behavior to $p_g$ when $r_N=p_g$, a penalizing term on the final tracking error is also included.
The cost function can be tailored for the robot at hand to favor certain behaviors.
For instance, a regularization term on (changes of) the control input can be included to provide a smoother trajectory. Moreover, a terminal cost to reach a desired final orientation can be incorporated.

The control law is given by
\begin{equation}
\label{eq:control_law}
    u = \bar{u}_{0}^*
\end{equation}
where $\bar{u}_{0}^*$ is extracted from the initial control input, $\mu_{0}^*$, of the optimal solution, $\bm{\mu}^*$, to \eqref{eq:mpc}. Although no explicit soft or hard constraints regarding the obstacles are used in the MPC formulation, obstacle avoidance is achieved as stated by the following theorem. This is obtained by ensuring a close tracking, $\lVert\varepsilon\rVert_2 < \rho$, of the path which is at least at a distance $\rho$ from any obstacle.
\begin{theorem}
\label{theorem:collision_avoidance}
The trajectory for a robot with dynamics \eqref{eq:robot_model} under control law \eqref{eq:control_law} is collision-free with respect to the obstacles $\mathcal{O}_k$, i.e. $p_k \in \mathcal{F}_k, \forall k$.
\end{theorem}
\begin{proof}
According to Assumption \ref{ass:quasi_static_obstacles} and \ref{ass:not_aggressive_obstacles} it suffices to show that $p_{k+1}\in \mathcal{F}_k$ at any sampling instance $k$.
First, we show existence of a solution to \eqref{eq:mpc} at all times. Define the trivial solution, $\bm{\mu}'$, as $\mu'_i=[\bar{u}'^T, 0]^T, \forall i \in \mathcal{N}$ where $\bar{u}'\in\mathcal{U} \textnormal{ s.t. } f(x_k,\bar{u}')=x_k$. Obviously, this satisfies constraint \eqref{eq:mpc_mu}. From \eqref{eq:mpc_init}-\eqref{eq:mpc_dyn} we get $z_{i+1} = [x^T_k, 0]^T, \forall i\in \mathcal{N}$. With current robot state $x_k\in\mathcal{X}$, we can conclude that \eqref{eq:mpc_z} is satisfied. Moreover, this gives $\varepsilon_{i+1} = r^0_k-p_k, \forall i\in \mathcal{N}$.
By construction, $r^0_k\in \mathcal{P}^0_k \subset \mathbb{B}(p_k, \rho_k)$ and it follows that $\varepsilon_{i+1} \in \mathcal{E}, \forall i\in\mathcal{N}$, satisfying constraint \eqref{eq:mpc_error}.
Hence, $\bm{\mu}'$ is a feasible solution.
Any solution, $\bm{\mu}$, to \eqref{eq:mpc} satisfies $\varepsilon_{i+1}\in \mathcal{E}, \forall i\in \mathcal{N}$. This implies $h(\bar{x}_{i+1})\in\mathcal{P}_k\oplus \mathbb{B}(\mathbf{0},\rho_k), \forall i\in \mathcal{N}$. Since $\mathcal{P}_k\subset\mathcal{F}^{\rho}_k$ we thus have $h(\bar{x}_1) \in \mathcal{F}_k$. Applying control law \eqref{eq:control_law}, $u_k=\bar{u}_0$, leads to $x_{k+1} = f(x_k,\bar{u}_0) = \bar{x}_1$ from \eqref{eq:mpc_init}-\eqref{eq:mpc_dyn}. Hence, $p_{k+1}=h(\bar{x}_1)\in\mathcal{F}_k$.
\end{proof}

The MPC works as a bridge to incorporate the robot constraints and find an admissible control sequence resulting in a smooth path within the clearance distance to the reference path. 
From this perspective, a short horizon is sufficient, e.g. 2-3 samples. However, when the orientation for a non-holonomic robot is poorly aligned with the reference path direction, the MPC may yield a solution where the robot is standing still if the constraints prohibits the robot from realigning sufficiently fast within the horizon. Hence, the horizon should be adapted for the admissible robot reorientation abilities, e.g. angular velocity limit for a unicycle robot.

\section{Implementation aspects\protect\footnote{Code is available at https://github.com/albindgit/starworld\_tunnel\_mpc.}}
\label{sec:implementation}
To implement the motion control scheme presented in Section \ref{sec:method} in a real-time application, possible adjustments are here presented. The adjustments aim at reducing the computational complexity and obtaining a fixed upper bound on computation time.

\subsubsection{Maximum workspace modification time}
For a fixed upper bound on computation time, Algorithm \ref{alg:obstacle_transformation} is terminated after a specified maximum time. If the algorithm is terminated prematurely before $\mathcal{O}^{\star}$ has been generated (before line \ref{l:O_star}), $\mathcal{O}^{\star}$ is set such that $\mathcal{F}^{\star}=\mathcal{F}^{\rho}$
where any non-starshaped polygon in $\mathcal{O}^{\rho}$ is divided into sub-obstacles by convex decomposition. In this way, $\mathcal{F}^{\star}$ is a star world and positively invariant for the dynamics \eqref{eq:reference_dynamics}. 

\subsubsection{Maximum reference generation time}
In practice, the reference path is computed by discrete integration of \eqref{eq:reference_dynamics} over an interval $s\in [0,N]$. The integration step size depend on current proximity to $\mathcal{O}^{\star}$ in order to guarantee that no obstacle boundary is penetrated.
Similar to the workspace modification, a maximum computation time is used to terminate the simulation when exceeded. In case the path generation is terminated prematurely, at $s' < N$, the path is extended with a static final position, i.e. $r(s) = r(s'), \forall s\in [s', N]$. 

\subsubsection{Reference buffering}
It is practical to reuse the computed reference path from previous sampling instance, appropriately shifted, if it is still collision-free and within a distance $\rho$ from the robot position, i.e. if $\mathcal{P}_{k-1}\subset \mathcal{F}^{\star}_k$ and $\mathcal{P}_{k-1}\cap\mathbb{B}(p_k,\rho_k)\neq \emptyset$.
The reference path integration is then initialized at the final position of $\mathcal{P}_{k-1}$. 

\subsubsection{Reference approximation}
To simplify calculation of $\varepsilon_i$, the discrete reference path, $\mathcal{P}$, is replaced by a function approximation, $\hat{r}(s)$, e.g. using polynomial regression. The approximation is biased to enforce $\hat{r}(0)=r(0)$ such that Theorem \ref{theorem:collision_avoidance} is not compromised. To account for approximation errors, the constraint for the tracking error, $\varepsilon_i=\hat{r}(s_i)-h(\bar{x}_i)$, is adjusted to $\mathcal{E} = \mathbb{B}(\mathbf{0}, \rho - \epsilon)$, where $\epsilon$ is the maximum approximation error
\begin{equation}
    \epsilon = \max_{s\in[0,N]} \lVert \hat{r}(s) - r(s) \rVert_2.
\end{equation}
It is assumed that a sufficiently flexible function representation is used such that $\epsilon < \rho$.

\section{Results}
We consider two scenarios: 1) a static scene with variations, and 2)  a moving obstacle in a static scene with corridors. The first one illustrates the workspace modification procedure and convergence while the second one illuminates the properties of the MPC. Both scenarios assume unicycle robot dynamics:
\begin{equation}
\label{eq:unicycle}
    f(x,u) = x + \begin{bmatrix}v\cos\theta\\
    v\sin\theta\\
    \omega\end{bmatrix}\Delta t,\quad h(x) = \begin{bmatrix}p^x\\p^y\end{bmatrix},
\end{equation}
where $x=[p^x,p^y,\theta]^T$ are the Cartesian position [m] and orientation [rad] of the robot. The control inputs $u=[v\ \omega]^T$ are the linear and angular velocities which are bounded by $[0, 1.5]$ m/s and $[-1.5, 1.5]$ rad/s, respectively. The sampling rate is $\Delta t = 0.2$s and as function approximation of the reference path, a polynomial of degree 10 is used.
For a smooth trajectory, the stage cost is extended with a regularization term on control input variation, $\sum_{i\in\mathcal{N}}\Delta\bar{u}^T_iR\Delta \bar{u}_i$. Here, $\Delta \bar{u}_i = \bar{u}_i-\bar{u}_{i-1}$ is the control variation, with $\bar{u}_{-1} = u_{k-1}$ being the previously applied control input, and where $R$ is a positively definite $2\times 2$ matrix. All numerical values for the control parameters are stated in Table \ref{tab:params}.

\begin{table}[!t]
\renewcommand{\arraystretch}{1.3}
\caption{Control Parameters}
\label{tab:params}
\centering
\begin{tabular}{c||c||c||c||c||c}
\hline
\bfseries $\bar{\rho}$ & $\gamma$ & $N$ & $c_s$ & $c_e$ & $R$\\
\hline\hline
0.3 & 0.5 & 5 & 500 & 100 & diag$(250, 2.5)$\\
\hline
\end{tabular}
\end{table}

In Fig. \ref{fig:obstacle_transformation} a static scene with five intersecting obstacles is shown. As seen in Fig. \ref{fig:sim_static_ch}-\ref{fig:sim_static_goal}, the resulting starshaped obstacle is depending on the robot position, goal position and other obstacles not in the combined cluster. When the starshaped obstacle can be extended with the convex hull, the path to the goal is more direct compared to if it is a concave obstacle. In each case, however, the guiding vector field has the goal as a global attractor. 

\def\figscale{0.48}
\begin{figure}
    \begin{subfigure}[t]{\figscale\linewidth}
        \includegraphics[width=\linewidth,trim={4cm 1cm 3.5cm 1.7cm},clip]{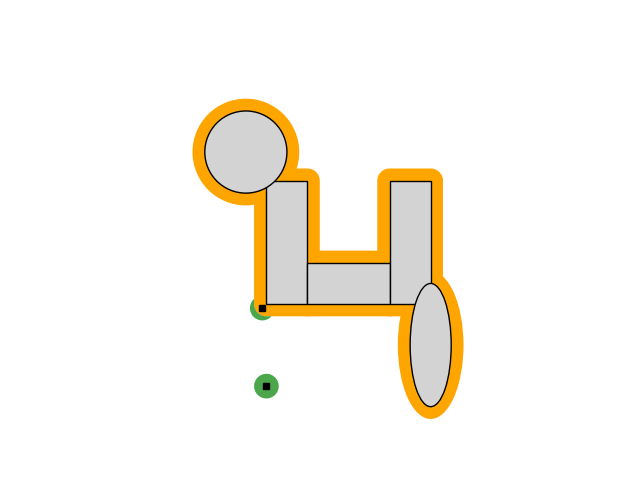}
         \caption{The initial reference set, $\mathcal{P}^0$, (green) for two robot positions.}
         \label{fig:sim_static_rho}
    \end{subfigure}
    \hfill
    \begin{subfigure}[t]{\figscale\linewidth}
        \includegraphics[width=\linewidth,trim={4cm 1cm 3.5cm 1.7cm},clip]{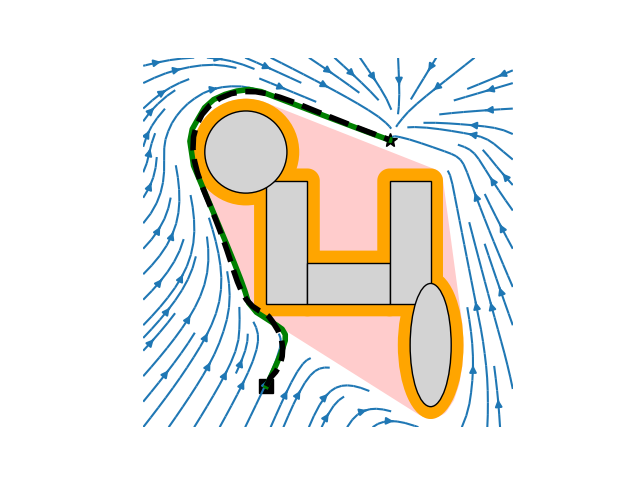}
         \caption{When possible, the cluster is extended with the convex hull.}
         \label{fig:sim_static_ch}
    \end{subfigure}
    \vskip 0.5\baselineskip
    \begin{subfigure}[t]{\figscale\linewidth}
        \includegraphics[width=\linewidth,trim={4cm 1cm 3.5cm 1.7cm},clip]{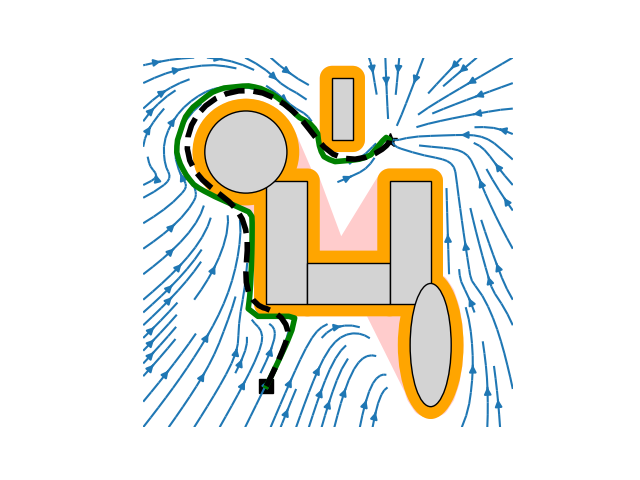}
         \caption{The cluster is not extended by the convex hull as this would lead to an intersection with the cluster external obstacle.}
         \label{fig:sim_static_external_obs}
    \end{subfigure}
    \hfill
    \begin{subfigure}[t]{\figscale\linewidth}
        \includegraphics[width=\linewidth,trim={4cm 1cm 3.5cm 1.7cm},clip]{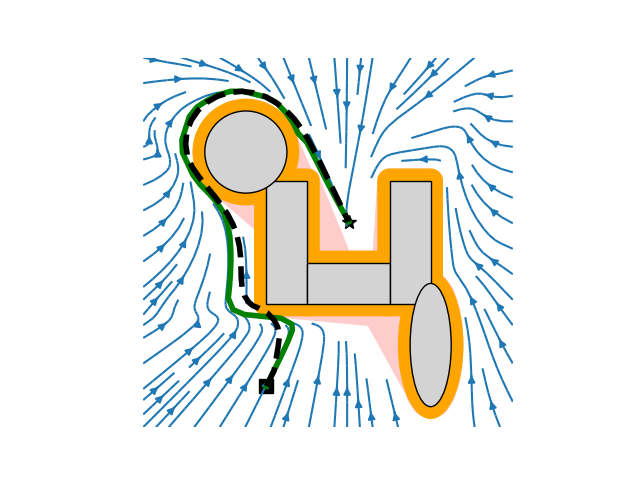}
         \caption{The cluster is adjusted to not include the goal point.}
         \label{fig:sim_static_goal}
    \end{subfigure}
    \caption{Each obstacle in $\mathcal{O}$ (grey) is dilated with $\rho$ to form $\mathcal{O}^{\rho}$ (yellow). Next, all intersecting obstacles are combined and extended with the starshaped hull (convex hull if possible) to form $\mathcal{O}^{\star}$ (red). The resulting robot path from initial position (square) to goal position (star) is shown as dashed black line. The vector field (blue arrows) and resulting path from initial position (green line) for the reference dynamics \eqref{eq:reference_dynamics} are also shown.}
    \label{fig:obstacle_transformation}
\end{figure}

In Fig. \ref{fig:sim_corridor} a scenario with three static polygons forming two corridors and a moving circular obstacle is depicted. The robot starts on one side of the corridors while the goal position is placed on the opposite side.
As seen, the robot initially enters the lower corridor to take the most direct path towards the goal but reroutes its path to the upper corridor as the lower gets blocked. In situations where the robot is well aligned with the reference path (e.g. Fig. \ref{fig:corridor_0_0} and \ref{fig:corridor_4_0}), the path speed can be set close to maximal such that $s_N\approx N$. That is, the last reference point, $\hat{r}(s_N)$, and thus the predicted robot position according to the coupling constraint \eqref{eq:mpc_error}, is close to the end of the reference path $\mathcal{P}$.
During the rotation in the lower corridor (Fig. \ref{fig:corridor_2_4}), the kinematics of the robot in combination with the tracking error constraint prevent the path speed to be maximal such that $s_N\ll N$, i.e. $\hat{r}(s_N)$ is not at the end of $\mathcal{P}$, and the robot is moving more slowly.

\begin{figure*}
        \begin{subfigure}[t]{0.24\textwidth}
                \includegraphics[width=\linewidth,trim={2.2cm 1.5cm 2.3cm 1.5cm},clip]{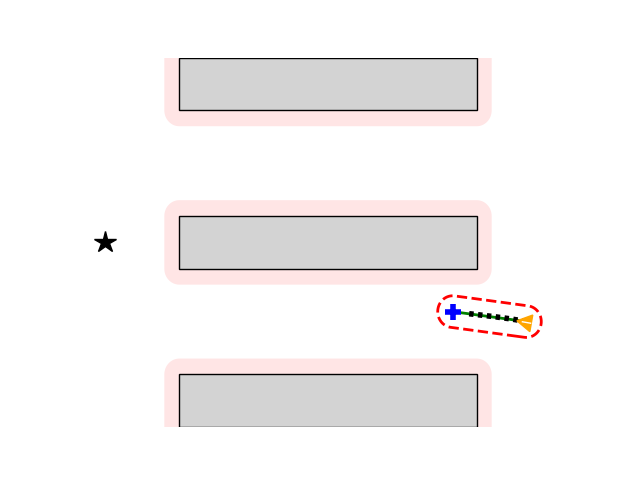}
                \caption{Time $t=0.0$: Both corridors are free and the robot starts moving according to the reference path into the lower corridor.}
                \label{fig:corridor_0_0}
        \end{subfigure}%
        \hfill
        \begin{subfigure}[t]{0.24\textwidth}
                \includegraphics[width=\linewidth,trim={2.2cm 1.5cm 2.3cm 1.5cm},clip]{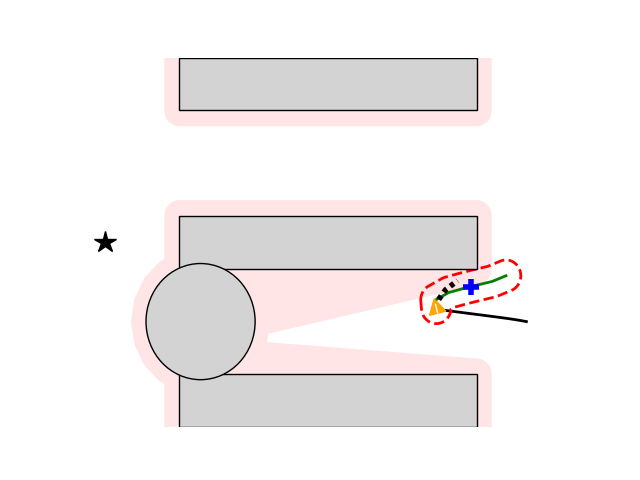}
                \caption{Time $t=2.2$: The moving obstacle has entered the lower corridor, blocking the way out. As the three intersecting obstacles are combined into a single starshaped obstacle, the reference path is now directed into the upper corridor.}
                \label{fig:corridor_2_4}
        \end{subfigure}%
        \hfill
        \begin{subfigure}[t]{0.24\textwidth}
                \includegraphics[width=\linewidth,trim={2.2cm 1.5cm 2.3cm 1.5cm},clip]{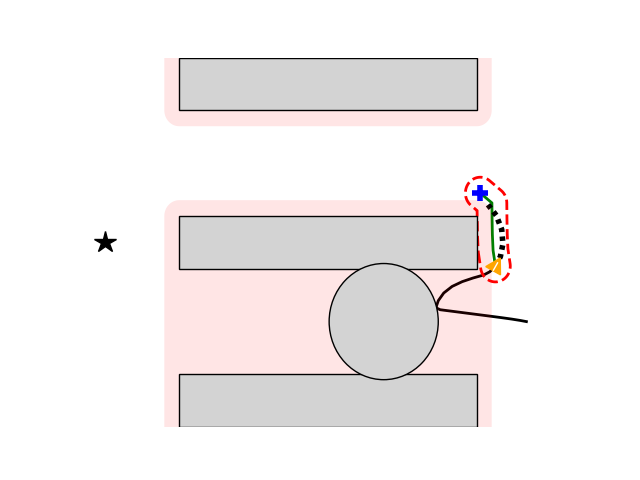}
                \caption{Time $t=4.2$: The robot has left the lower corridor and the combined starshaped obstacle is made convex since neither the robot nor goal position are within this convex hull.}
                \label{fig:corridor_4_0}
        \end{subfigure}%
        \hfill
        \begin{subfigure}[t]{0.24\textwidth}
                \includegraphics[width=\linewidth,trim={2.2cm 1.5cm 2.3cm 1.5cm},clip]{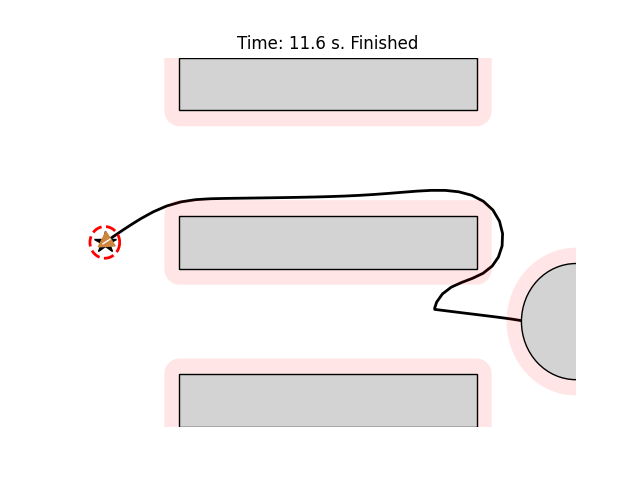}
                \caption{Time $t=11.2$: The robot has arrived at the goal position. When the moving obstacle has left the corridor, all obstacles are again treated separately.}
                \label{fig:corridor_11_0}
        \end{subfigure}
        \caption{A unicycle robot in a double corridor scenario with a circular moving obstacle entering one corridor. The obstacles, $\mathcal{O}$, are shown as grey regions, the dilated starshaped obstacles, $\mathcal{O}^{\star}$, are shown as red regions, the goal position is depicted as black star and the yellow triangle indicates the current robot position and orientation. The reference path, $\mathcal{P}$, is shown as green line with boundary of accompanying tunnel region, $\mathcal{P}\oplus \mathcal{E}$, depicted with dashed red line. The final reference point in the MPC horizon, $\hat{r}(s_N)$, is shown with blue cross and the predicted robot trajectory for the MPC solution, $h(\bar{x}_i),\ i\in \mathcal{N}$, is shown as dotted black line. The travelled path is shown as black solid line.}
        \label{fig:sim_corridor}
\end{figure*}

\begin{figure}
\begin{subfigure}[t]{0.48\linewidth}
        \includegraphics[width=\linewidth,trim={0.4cm 0.5cm 1.5cm 1.6cm},clip]{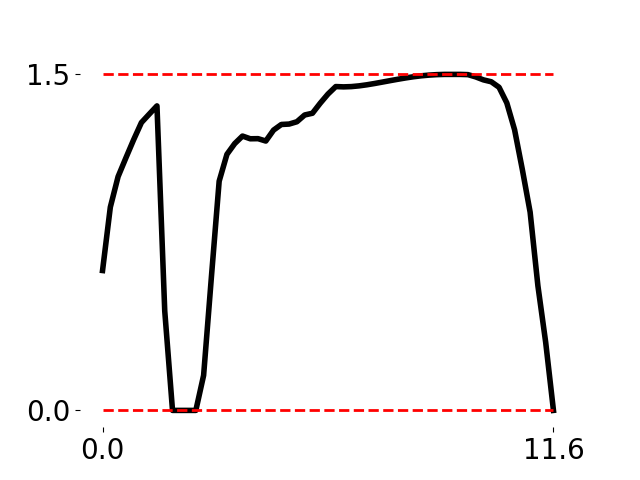}
     \caption{.}
     \label{fig:corridor_u1}
    \end{subfigure}
    \hfill
    \begin{subfigure}[t]{0.48\linewidth}
        \includegraphics[width=\linewidth,trim={0.4cm 0.5cm 1.5cm 1.6cm},clip]{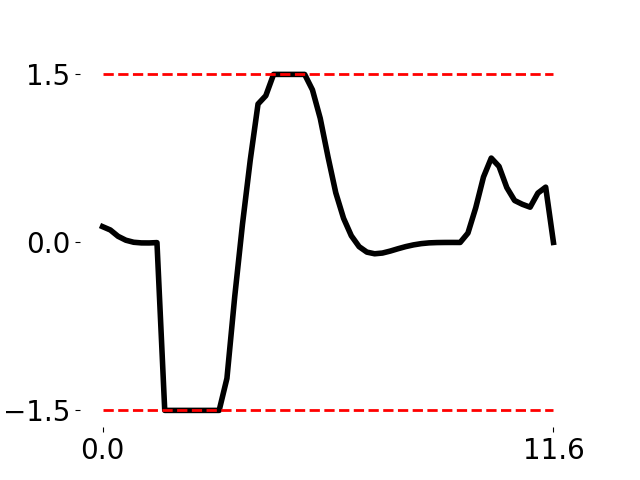}
         \caption{.}
         \label{fig:corridor_u2}
    \end{subfigure}
    \caption{The applied control signals in the double corridor scenario. First, the robot accelerates until the lower corridor is blocked, at which point it decelerates to full stop and completes a U-turn. Next, it turns around the center wall into the upper corridor and accelerates to and maintains full speed. Finally, it decelerates to reach the final goal position.}
    \label{fig:corridor_u}
\end{figure}

\section{Conclusion}
%In this paper we 
This paper proposed a motion control scheme for robots operating in a Cartesian plane containing a collection of dynamic, possibly intersecting, obstacles. The method combines previous work by the authors to adjust the workspace into a scene of disjoint obstacles with a closed form dynamical system formulation to generate a receding horizon path. An MPC controller is used to compute an admissible control sequence yielding attracting behavior towards the goal with close enough path tracking to ensure collision avoidance.

As the method relies on conservative treatment of the obstacle regions, a natural drawback is the possible gap closing in narrow passages. While conditions for global convergence of the reference path can be derived, it might be lost by use of the MPC depending on the horizon and robot constraints.
In future work, we aim at achieving convergence guarantees for the full motion control scheme. In the current formulation, environment boundaries are not considered. As robots commonly are restricted to some specific area, e.g. a room or building, it is of interest to include specification of a bounded workspace as well.

\bibliographystyle{ieeetr}
\bibliography{references}

\end{document}

%% file: figures/architecture.tex
\begin{tikzpicture}
    % Workspace modification
    \node [draw,
        text width=8em, fill=blue!20, text centered,
        minimum height=5em, rounded corners
    ]  (obs_trans) at (0,0) {Workspace\\ modification};
    % Path generator
    \node [draw,
        text width=8em, fill=blue!20, text centered, 
        minimum height=5em, rounded corners,
        right=5em of obs_trans
    ] (path_gen) {Reference path\\ generation};
    % Motion controller 
    \node [draw,
        text width=8em, fill=blue!20, text centered, 
        minimum height=5em, rounded corners,
        right=3em of path_gen
    ] (controller) {Tunnel-following MPC};
    % Robot block
    \node [draw,
        text width=5em, fill=white!20, text centered, 
        minimum height=3em, rounded corners,
        below=3em of path_gen
    ]  (robot) {Robot};
    % Scene input
    \node [left=2em of obs_trans](scene){};
    \node [below=1em of path_gen](x){};
    \node [right=1em of controller](u){};
    \node [above=1.5em of controller](rho){};
    % Arrows with text label
    \draw[-stealth] (scene.center) -- (obs_trans.west) 
        node[near start,above]{$\mathcal{O}, p^g$};
    \draw[-stealth] (obs_trans.east) -- (path_gen.west) 
        node[midway,above,text width=5em,text centered]{$\mathcal{O}^{\star}, r^0, r^g$};
    \draw[-stealth] (path_gen.east) -- (controller.west) 
        node[midway,above]{$\mathcal{P}$};
    \draw[-] (controller.east) -- (u.center)
        node[near end,above]{$u$};
    \draw[-stealth] (u.center) |- (robot.east);
    % \draw[-] (controller.east) -- ++ (.5,0) 
    %     node[](u_node){}node[midway,above]{$u$};
    \draw[-] (robot.north) -- (x.center)
        node[midway,left]{$x$};
    \draw[-stealth] (x.center) -| (controller.south);
    \draw[-stealth] (x.center) -| (obs_trans.south);
    \draw[-] (obs_trans.north) |- (rho.center)
        node[near start,right]{$\rho$};
    \draw[-stealth] (rho.center) -- (controller.north);
    % \draw[-] (workspace.west) -| (scene_node.center) 
        % node[near start,above]{$\mathcal{O}, x, x_g$};
    \end{tikzpicture}

%% file: main.bbl
\begin{thebibliography}{10}

\bibitem{khatib_85}
O.~Khatib, ``Real-time obstacle avoidance for manipulators and mobile robots,''
  in {\em Proceedings. 1985 IEEE International Conference on Robotics and
  Automation}, vol.~2, pp.~500--505, 1985.

\bibitem{ginesi_etal_19}
M.~Ginesi, D.~Meli, A.~Calanca, D.~Dall'Alba, N.~Sansonetto, and P.~Fiorini,
  ``Dynamic movement primitives: Volumetric obstacle avoidance,'' in {\em 2019
  19th International Conference on Advanced Robotics (ICAR)}, pp.~234--239,
  2019.

\bibitem{stavridis_etal_17}
S.~Stavridis, D.~Papageorgiou, and Z.~Doulgeri, ``Dynamical system based
  robotic motion generation with obstacle avoidance,'' {\em IEEE Robotics and
  Automation Letters}, vol.~2, no.~2, pp.~712--718, 2017.

\bibitem{rimon_koditschek_92}
E.~Rimon and D.~Koditschek, ``Exact robot navigation using artificial potential
  functions,'' {\em IEEE Transactions on Robotics and Automation}, vol.~8,
  no.~5, pp.~501--518, 1992.

\bibitem{loizou_11_2}
S.~G. Loizou, ``Closed form navigation functions based on harmonic
  potentials,'' in {\em 2011 50th IEEE Conference on Decision and Control and
  European Control Conference}, pp.~6361--6366, 2011.

\bibitem{paternain_etal_18}
S.~{Paternain}, D.~E. {Koditschek}, and A.~{Ribeiro}, ``Navigation functions
  for convex potentials in a space with convex obstacles,'' {\em IEEE
  Transactions on Automatic Control}, vol.~63, no.~9, pp.~2944--2959, 2018.

\bibitem{connolly_etal_90}
C.~Connolly, J.~Burns, and R.~Weiss, ``Path planning using \text{L}aplace's
  equation,'' in {\em Proceedings., IEEE International Conference on Robotics
  and Automation}, pp.~2102--2106 vol.3, 1990.

\bibitem{feder_slotine_97}
H.~Feder and J.-J. Slotine, ``Real-time path planning using harmonic potentials
  in dynamic environments,'' in {\em Proceedings of International Conference on
  Robotics and Automation}, vol.~1, pp.~874--881 vol.1, 1997.

\bibitem{daily_bevly_08}
R.~Daily and D.~M. Bevly, ``Harmonic potential field path planning for high
  speed vehicles,'' in {\em 2008 American Control Conference}, pp.~4609--4614,
  2008.

\bibitem{huber_etal_19}
L.~Huber, A.~Billard, and J.-J. Slotine, ``Avoidance of convex and concave
  obstacles with convergence ensured through contraction,'' {\em IEEE Robotics
  and Automation Letters}, vol.~4, no.~2, pp.~1462--1469, 2019.

\bibitem{huber_etal_22}
L.~Huber, J.-J. Slotine, and A.~Billard, ``Avoiding dense and dynamic obstacles
  in enclosed spaces: Application to moving in crowds,'' {\em IEEE Transactions
  on Robotics}, pp.~1--10, 2022.

\bibitem{dahlin_karayiannidis_22}
A.~Dahlin and Y.~Karayiannidis, ``Creating star worlds: Reshaping the robot
  workspace for online motion planning,'' March 2023.
\newblock arXiv:2205.09336 [cs.RO].

\bibitem{schulman_etal_14}
J.~Schulman, Y.~Duan, J.~Ho, A.~Lee, I.~Awwal, H.~Bradlow, J.~Pan, S.~Patil,
  K.~Goldberg, and P.~Abbeel, ``Motion planning with sequential convex
  optimization and convex collision checking,'' {\em The International Journal
  of Robotics Research}, vol.~33, no.~9, pp.~1251--1270, 2014.

\bibitem{zhang_etal_21}
X.~Zhang, A.~Liniger, and F.~Borrelli, ``Optimization-based collision
  avoidance,'' {\em IEEE Transactions on Control Systems Technology}, vol.~29,
  no.~3, pp.~972--983, 2021.

\bibitem{brito_etal_19}
B.~Brito, B.~Floor, L.~Ferranti, and J.~Alonso-Mora, ``Model predictive
  contouring control for collision avoidance in unstructured dynamic
  environments,'' {\em IEEE Robotics and Automation Letters}, vol.~PP,
  pp.~1--1, 07 2019.

\bibitem{sanchez_etal_21}
I.~S\'{a}nchez, A.~D’Jorge, G.~V. Raffo, A.~H. Gonz\'{a}lez, and
  A.~Ferramosca, ``Nonlinear model predictive path following controller with
  obstacle avoidance,'' {\em Journal of Intelligent and Robotic Systems},
  vol.~102, may 2021.

\bibitem{ji_etal_16}
J.~ji, A.~Khajepour, W.~Melek, and Y.~Huang, ``Path planning and tracking for
  vehicle collision avoidance based on model predictive control with
  multiconstraints,'' {\em IEEE Transactions on Vehicular Technology}, vol.~66,
  pp.~1--1, 01 2016.

\bibitem{faulwasser_findeisen_16}
T.~Faulwasser and R.~Findeisen, ``Nonlinear model predictive control for
  constrained output path following,'' {\em IEEE Transactions on Automatic
  Control}, vol.~61, no.~4, pp.~1026--1039, 2016.

\bibitem{vanduijkeren_19}
N.~van Duijkeren, {\em Online Motion Control in Virtual Corridors - for Fast
  Robotic Systems}.
\newblock Phd thesis, KU Leuven, [Online]. Available:
  https://lirias.kuleuven.be/retrieve/527169, 2019.

\bibitem{hansen_etal_20}
G.~Hansen, I.~Herburt, H.~Martini, and M.~Moszyńska, ``Starshaped sets,'' {\em
  Aequationes mathematicae}, vol.~94, 12 2020.

\end{thebibliography}
